\documentclass{article}

\PassOptionsToPackage{square,numbers,sort}{natbib}

\usepackage[preprint]{neurips_2026}

\usepackage[utf8]{inputenc} 
\usepackage[T1]{fontenc}
\usepackage{hyperref}
\usepackage{url}
\usepackage{booktabs}
\usepackage{amsfonts}
\usepackage{nicefrac}
\usepackage{microtype}
\usepackage{xcolor}
\usepackage{soul}
\usepackage{tabularx}
\usepackage{amsmath}
\usepackage{amssymb}
\usepackage{amsthm}

\newtheorem{theorem}{Theorem}
\newtheorem{proposition}[theorem]{Proposition}
\newtheorem{remark}{remark}
\usepackage{graphicx}
\usepackage{tablefootnote}
\usepackage{caption}
\usepackage[table]{xcolor}

\usepackage{multirow}
\usepackage{pgfplots}
\pgfplotsset{compat=1.18}
\usepackage{subcaption} 
\usepackage{wrapfig}
\usepackage{paracol}
\usepackage{geometry}
\usepackage{tikz}
\title{\emph{Beyond Position Bias}: Shifting Context Compression from Position-Driven to Semantic-Driven}

\author{%
  Jiwei Tang$^{1}$\thanks{Equal Contribution. $^{\ddag}$ Corresponding Author.},  Zhijing Huang$^{1*}$, Xinyu Zhang$^{1}$, Chen Jason Zhang$^{2}$, \\ \bf Jianxing Yu$^{1}$, Libin Zheng$^{1{\ddag}}$, Rui Meng$^{3}$, Jian Yin$^{1}$\\
 $^{1}$Sun Yat-sen University \quad $^{2}$Hong Kong Polytechnic University \\
$^{3}$Beijing Normal–Hong Kong Baptist University \\
\texttt{jiweitang23@gmail.com} \quad \texttt{zhenglb6@mail.sysu.edu.cn}
}

\begin{document}

\maketitle

\begin{abstract}
Large Language Models (LLMs) have demonstrated exceptional performance across diverse tasks. However, their deployment in long-context scenarios faces high computational overhead and information redundancy. While soft prompt compression has emerged as a promising way to mitigate these costs by compressing sequences into compact embeddings, existing paradigms remain fundamentally constrained by \textbf{position bias}: they primarily rely on learnable tokens insertion at fixed positions or group tokens according to their physical token layout, thereby inducing performance instability and semantic fragmentation. To overcome this bottleneck, we propose \textbf{Se}mantic \textbf{Co}nsistency Context Compression (\textbf{SeCo}), a method that shifts context compression from \emph{position-driven} to \emph{semantic-driven}. Rather than constraint by physical token layout, SeCo dynamically anchors compression \emph{directly} in the semantic space by selecting query-relevant tokens as semantic centers and aggregating remaining tokens via consistency-weighted merging. This design inherently preserves semantic consistency while eliminating position bias. Extensive experiments on 14 benchmarks across two backbone models demonstrate that SeCo consistently shows superiority in downstream tasks, inference latency, and out-of-domain robustness. The code is available at \url{https://anonymous.4open.science/r/seco-EE5E}.
\end{abstract}

\definecolor{lightblue}{rgb}{0.92, 0.95, 1.0}
\sethlcolor{lightblue}
\section{Introduction}

Large Language Models (LLMs) have demonstrated remarkable capabilities across a wide range of tasks, catalyzing a paradigm shift in Natural Language Processing (NLP)~\citep{DBLP:journals/corr/abs-2407-10671,DBLP:journals/corr/abs-2501-12948,kimiteam2025kimik2openagentic}. Real-world applications increasingly demand longer prompt inputs, such as Retrieval-Augmented Generation (RAG)~\citep{DBLP:conf/nips/LewisPPPKGKLYR020}, In-Context Learning (ICL)~\citep{DBLP:conf/emnlp/Dong0DZMLXX0C0S24}, and the deployment of LLMs as agents for long-horizon task planning~\citep{DBLP:journals/corr/abs-2510-11967}. Despite the enhanced contextual coverage afforded by long sequences, deploying LLMs directly on such inputs faces (1) significant information redundancy~\citep{liu2024forgettingcurvereliablemethod,DBLP:conf/acl/JiangWL0L0Q24,DBLP:conf/iclr/00010WWCW24,DBLP:conf/naacl/TangXLZZHZ25,DBLP:journals/corr/abs-2505-12215} and (2) high computational load caused by Transformer's quadratic time complexity~\citep{DBLP:conf/nips/VaswaniSPUJGKP17,DBLP:journals/corr/abs-2602-01840,DBLP:journals/corr/abs-2602-01719} on its self attention mechanism, limiting their real-world utility.

\begin{figure}[thb]
    \centering
    \begin{subfigure}[b]{0.3\textwidth}
      \hspace*{-0.4cm}
      \includegraphics[height=4cm, keepaspectratio]{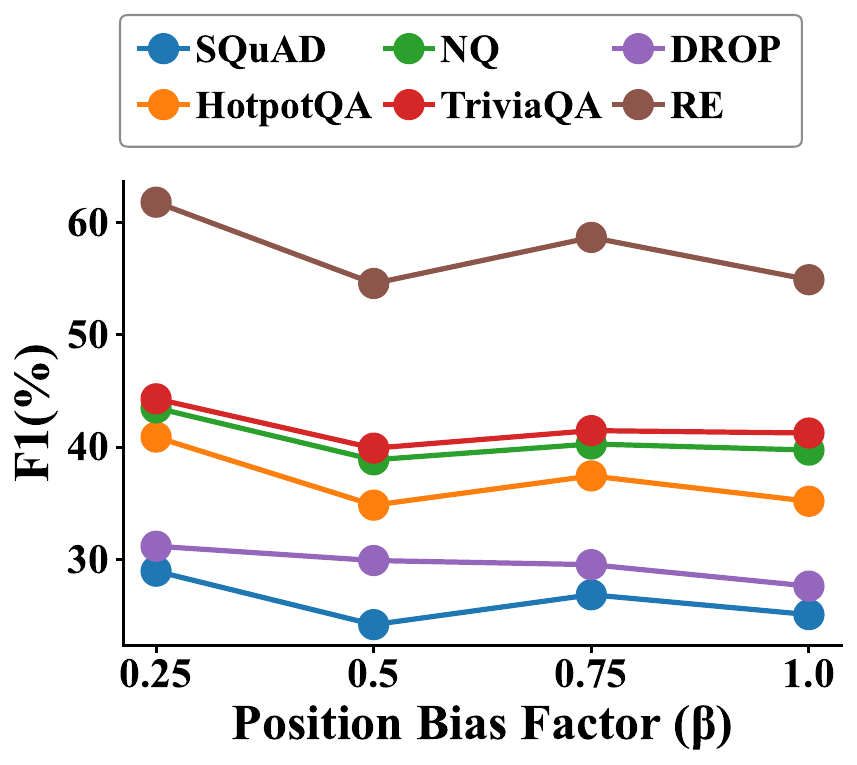}
      \caption{Positional Encoding Bias}
      \label{fig:pos-encoding-bias}
    \end{subfigure}
    \begin{subfigure}[b]{0.3\textwidth}
      \centering
      \hspace*{-0.15cm}
      \includegraphics[height=4cm, keepaspectratio]{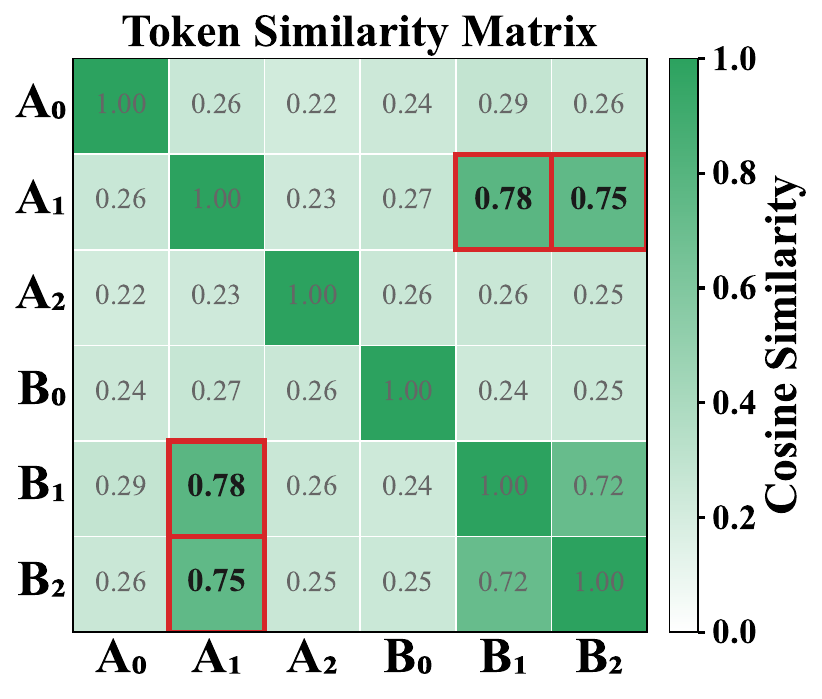}
      \caption{Group-based Position Bias}
      \label{fig:group-based-position-Bias}
    \end{subfigure}
    \begin{subfigure}[b]{0.3\textwidth}
      \centering
      \hspace*{0.2cm}
      \includegraphics[height=4cm, keepaspectratio]{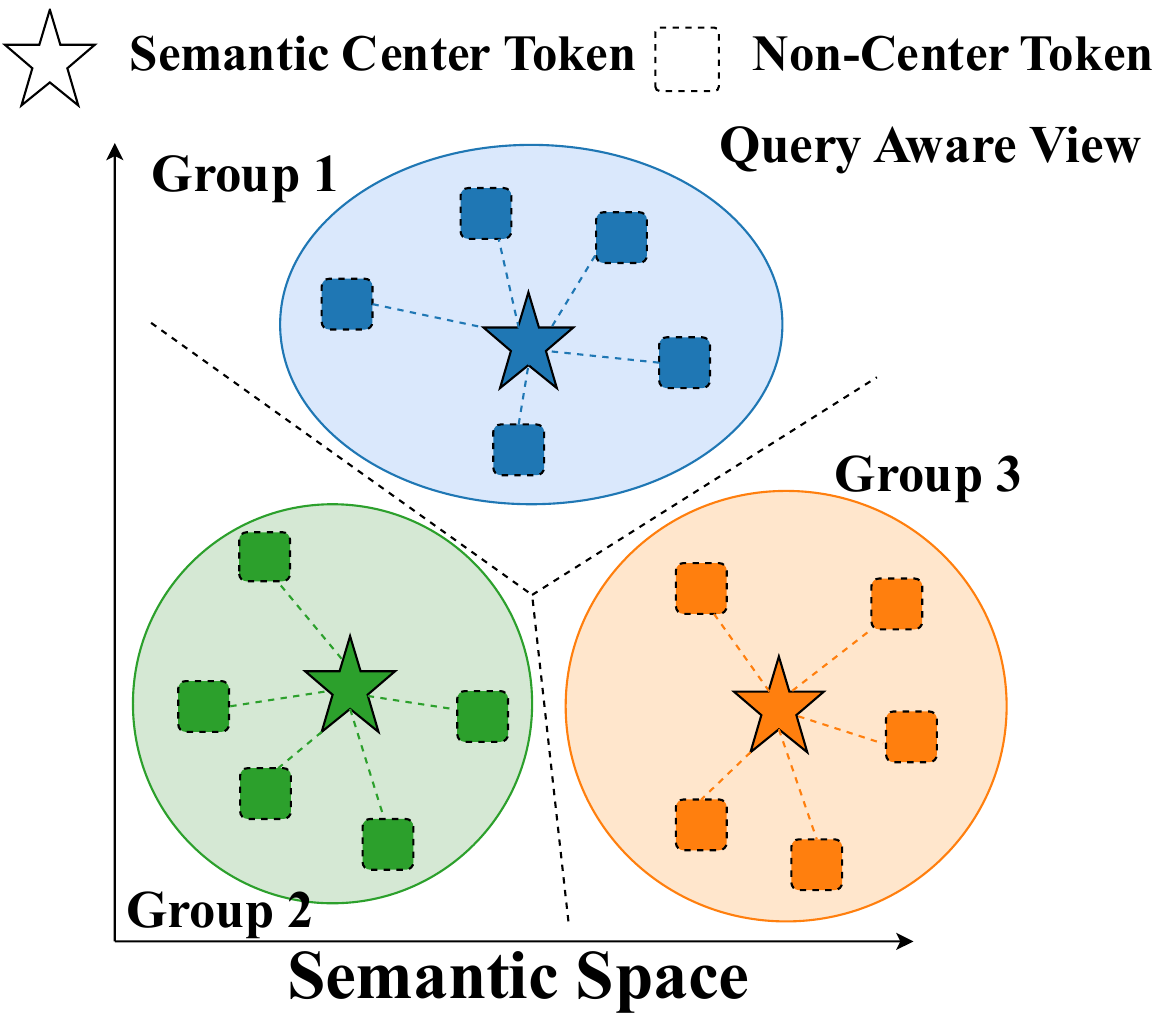}
      \caption{SeCo}
      \label{fig:Intro-SeCo}
    \end{subfigure}
    \caption{Comparison between existing methods and SeCo. 
    (a) \textbf{Performance Instability.} Learnable token insertion method, such as ICAE~\citep{DBLP:conf/iclr/00010WWCW24}, exhibit pronounced sensitivity to the position bias factor $\beta$, which controls the positional encoding of learnable tokens (as detailed in Figure~\ref{fig:Position_Encoding_Bias}). Consequently, F1 scores fluctuate substantially across various benchmarks under different $\beta$ configurations.
    (b) \textbf{Semantic Fragmentation.} Token $A_1$ in Group A is semantically highly correlated with tokens $B_1$ and $B_2$ in Group B. However, position-based grouping rigidly partitions these semantically related tokens (highlighted in red boxes) into separate compression groups.
    (c) \textbf{Semantic Consistency Compression (SeCo).} SeCo selects the top-$k$ tokens most relevant to the query as semantic centers. Subsequently, remaining tokens are dynamically assigned and merged \emph{directly} in the semantic space based on query-centric semantic consistency.}
    \label{fig:seco_intro}
\end{figure}

Prompt compression~\citep{DBLP:conf/naacl/LiLSC25} has emerged as a promising way to shorten sequence lengths and mitigate redundancy. Existing methods can be mainly categorized into hard prompt compression~\citep{DBLP:conf/acl/JiangWL0L0Q24,DBLP:conf/acl/PanWJXLZLR0LZQ024,DBLP:conf/naacl/TangXLZZHZ25,DBLP:conf/iclr/ChirkovaFNC25,DBLP:conf/emnlp/LuoCL25,DBLP:conf/acl/HwangCJSHP25,DBLP:journals/corr/abs-2510-08907} and soft prompt compression~\citep{DBLP:conf/iclr/00010WWCW24,DBLP:conf/iclr/Zhang0XSYD25,DBLP:conf/acl/LiSC25,DBLP:conf/acl/ChenLXZWHSSTL025,DBLP:conf/acl/DaiLHZZWXL25,10.1145/3701551.3703527,DBLP:journals/corr/abs-2505-12215,DBLP:journals/corr/abs-2602-01840,DBLP:journals/corr/abs-2602-01719}. Hard compression methods typically discard non-selected tokens entirely, failing to leverage secondary contextual information and compromising semantic coherence, which limits information utility~\citep{DBLP:journals/corr/abs-2508-15813}. In contrast, soft prompt compression exploits inherent redundancy in the semantic space~\citep{DBLP:conf/emnlp/Ethayarajh19,DBLP:conf/acl/AghajanyanGZ20} by learning compact soft embeddings, typically preserving richer semantic information, especially under high compression rates~\citep{DBLP:journals/corr/abs-2505-12215}. However, we identify a critical bottleneck: existing soft compression methods, which can be categorized into learnable token insertion~\citep{DBLP:conf/iclr/00010WWCW24,DBLP:conf/iclr/Zhang0XSYD25,DBLP:conf/acl/LiSC25,DBLP:conf/acl/ChenLXZWHSSTL025,DBLP:conf/acl/DaiLHZZWXL25,10.1145/3701551.3703527} and position-based grouping~\citep{DBLP:conf/acl/CaoCLPHCS24,DBLP:journals/corr/abs-2505-12215,DBLP:journals/corr/abs-2602-01840,DBLP:journals/corr/abs-2602-01719}, rely heavily on positional priors, introducing \textbf{position bias}.

Specifically, current soft prompt compression methods suffer from two systematic forms of \textbf{position bias}: (1) \textbf{Positional encoding bias.} As illustrated in Figure~\ref{fig:pos-encoding-bias}, methods employing learnable token insertion (e.g., ICAE~\citep{DBLP:conf/iclr/00010WWCW24}) exhibit significant instability in downstream tasks depending on the Positional Encodings (PEs) assigned to these tokens. This sensitivity to perturbations in PEs~\citep{DBLP:conf/nips/VaswaniSPUJGKP17,DBLP:journals/ijon/SuALPBL24} aligns with recent observations in~\cite{DBLP:conf/emnlp/ZhaoLLHXXZ25}. (2) \textbf{Group-based position bias.} The principle of semantic consistency posits that semantically coherent vector structures facilitate more effective compression~\citep{DBLP:journals/tit/GrayN98, DBLP:books/daglib/0016881, DBLP:journals/tit/Donoho06}. However, position-based grouping methods \emph{rely on the physical token layout}, potentially partitioning semantically related tokens into distinct compression groups (see Figure~\ref{fig:group-based-position-Bias}). This induces semantic fragmentation, thereby violating semantic consistency. Furthermore, due to the inherent sparsity of task-relevant information~\citep{DBLP:conf/emnlp/0001DGL23,DBLP:conf/acl/JiangWL0L0Q24,DBLP:conf/naacl/TangXLZZHZ25}, the correlation between such derived groups and the query is often highly imbalanced (i.e., certain groups are entirely irrelevant to the query). These findings collectively highlight a core limitation: \textit{over-reliance on positional priors constrains compression effectiveness.} This motivates our research: \textit{How can we mitigate the negative effects of position bias in soft prompt compression?}

To address this, we draw inspiration
from cognitive science: human cognition is fundamentally organized as a semantic network rather than a linear sequence, where concepts are linked by meaning rather than physical proximity~\citep{semantic_processing,word_concept}. Motivated
by this, we shift the compression paradigm from \textit{position-driven} to \textit{semantic-driven}, proposing the \textbf{Se}mantic Consistency \textbf{Co}ntext Compression (\textbf{SeCo}) framework (Figure~\ref{fig:Intro-SeCo}). SeCo eliminates the need for injecting learnable tokens at fixed positions, instead \emph{modeling context structure directly within the semantic space to avoid positional encoding bias}. To resolve semantic fragmentation caused by group-based postion bias, SeCo reconstructs the grouping logic guided by query semantics: we first compute the relevance between context tokens and the query to identify the top-$k$ most relevant tokens as semantic centroids. Subsequently, remaining tokens are dynamically assigned to these centroids based on semantic consistency, and compressed representations are generated via consistency-weighted aggregation within each center. \textit{This process is driven by semantic relevance rather than physical token layout; each token autonomously identifies its suitable compression group directly in the semantic space, effectively mitigating group-based position bias.}

Our contributions are summarized as follows: (1) We identify and analyze position bias as a bottleneck that fundamentally limits existing soft prompt compression methods. (2) We propose SeCo, a semantic-driven compression framework designed to mitigate position bias and enhance semantic consistency. (3) We conduct extensive experiments across two backbone models (i.e., LLaMA and Qwen) and 14 benchmarks (e.g., long context question-answering and summarization), including both in-domain and out-of-domain evaluations. Experiment results demonstrate that SeCo consistently delivers superior performance in terms of downstream tasks, computational efficiency, and robustness.
\section{Related Work}
\label{sec:related_work}

\paragraph{Hard Prompt Compression.}
Hard prompt compression aims to reduce input length by explicitly shortening the original prompt, either through removing less informative tokens or generating condensed textual summaries. The resulting compressed prompt remains in the form of discrete natural language. Existing approaches can be broadly grouped into four categories~\citep{DBLP:journals/corr/abs-2505-12215}. (1) Perplexity-based methods estimate token importance according to perplexity and typically adopt a coarse-to-fine compression strategy to progressively discard less critical content, as exemplified by LLMLingua~\citep{DBLP:conf/emnlp/JiangWLYQ23}, LongLLMLingua~\citep{DBLP:conf/acl/JiangWL0L0Q24}, and Perception Compressor~\citep{DBLP:conf/naacl/TangXLZZHZ25}. (2) Bidirectional semantic-based methods address the limitations of the unidirectional perplexity signal by incorporating bidirectional semantic information, as in LLMLingua-2~\citep{DBLP:conf/acl/PanWJXLZLR0LZQ024}, MOOSComp~\citep{DBLP:journals/corr/abs-2504-16786}, Provence~\citep{DBLP:conf/iclr/ChirkovaFNC25}, and SAC~\citep{DBLP:journals/corr/abs-2510-08907}. (3) Attention-based methods exploit the intrinsic attention patterns of LLMs to identify and retain salient information during compression, such as AttnComp~\citep{DBLP:conf/emnlp/LuoCL25}. (4) Summary-generation methods compress long contexts by producing concise linguistic summaries that preserve task-relevant information, as demonstrated by CompACT~\citep{DBLP:conf/emnlp/YoonLHJK24}, RECOMP~\citep{DBLP:conf/iclr/XuSC24}, and EXIT~\citep{DBLP:conf/acl/HwangCJSHP25}. \emph{Despite their effectiveness in improving inference efficiency, these hard compression methods inevitably operate in a lossy manner and may impair the semantic coherence of the original prompt.}

\paragraph{Soft Prompt Compression.} Compared with hard prompt compression methods that rely on discrete selection strategies, soft prompt compression compress context by distilling contextual information into a compact continuous embedding space. This line of research has attracted considerable attention in recent years. Some methods, such as GIST~\citep{DBLP:conf/nips/Mu0G23}, AutoCompressor~\citep{DBLP:conf/emnlp/ChevalierWAC23}, ICAE~\citep{DBLP:conf/iclr/00010WWCW24}, LLoCO~\citep{DBLP:conf/emnlp/TanLPWZK0P24}, Beacon~\citep{DBLP:conf/iclr/Zhang0XSYD25}, 500xCompressor~\citep{DBLP:conf/acl/LiSC25}, typically obtain compact compressed representations by introducing appending learnable tokens. Other methods, including QGC~\citep{DBLP:conf/acl/CaoCLPHCS24}, GMSA~\citep{DBLP:journals/corr/abs-2505-12215}, RAM~\citep{DBLP:journals/corr/abs-2602-01840}, and COMI~\citep{DBLP:journals/corr/abs-2602-01719}, instead group tokens according to their physical token layout, thereby learning compressed representations. While these methods achieve relative rich semantic density, they are constrained by \textit{position priors}. Whether through the injection of learnable tokens at fixed positions or grouping based on physical token layout, such reliance introduces \textbf{position bias} that leads to performance instability or semantic fragmentation. \emph{In contrast, SeCo avoids position constraints by selecting semantic centroids within the semantic space and allocating remaining tokens based on semantic consistency for compression, thereby effectively mitigating position bias.}
\section{Method}
\label{sec:method}

\begin{figure}
\centering
\includegraphics[width=0.8\linewidth]{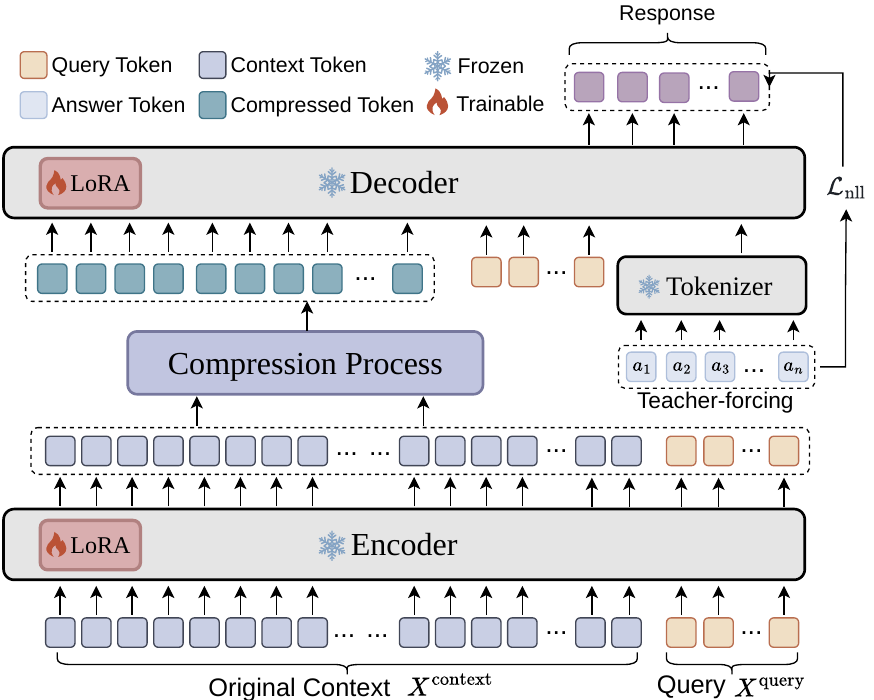}
\caption{The Training Paradigm of SeCo. SeCo is based on an encoder-decoder architecture. The original context $X^{\text{context}}$ and query $X^{\text{query}}$ are first encoded into hidden states, which are then compressed via a compression process (see Figure~\ref{fig:SeCo}). During training, the encoder and decoder are fine-tuned with Low-Rank Adaptation (LoRA)~\citep{DBLP:conf/iclr/HuSWALWWC22}.} 
\label{fig:FT_Training}
\end{figure}
\begin{figure}
\centering
\includegraphics[width=0.95\linewidth]{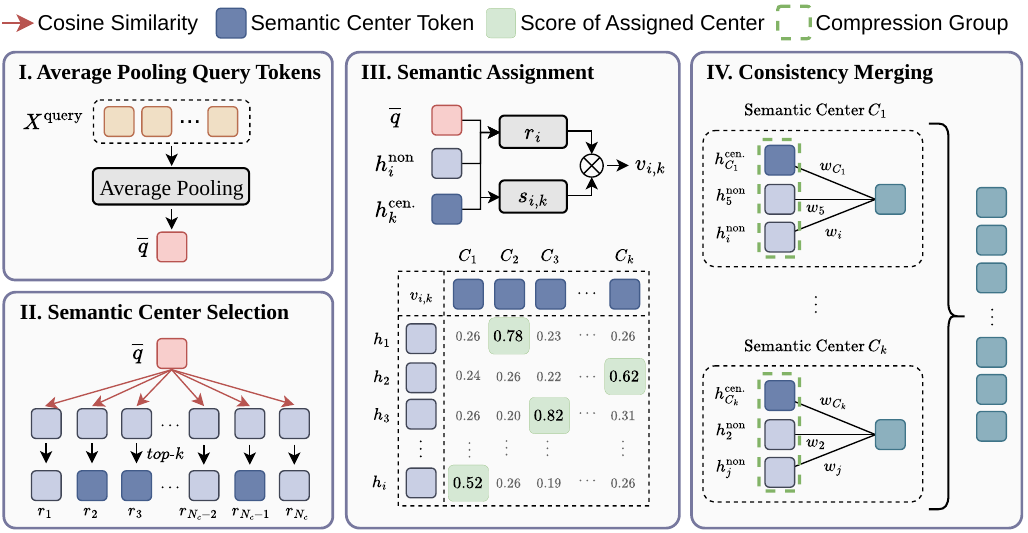}
\caption{The Compression Process of SeCo. Specifically, it sequentially performs four steps: \textbf{I. Average Pooling Query Tokens.} Obtain a single query representation $\bar{q}$ via average pooling; \textbf{II. Semantic Center Selection.} The \textit{top-k} tokens are selected as semantic center tokens based on the cosine similarity  between each context token and $\bar{q}$; \textbf{III. Semantic Assignment.} $h^{\text{cen.}}$ represents the selected semantic centers while $h^{\text{non}}$ denotes the remaining tokens. Each non-center token is assigned to the most relevant semantic center based on $v$ (defined in Eq.~\ref{eq:value}), where $v$ is the product of its cosine similarity with the semantic center and $\bar{q}$; \textbf{IV. Consistency Merging.} Each compression group is merged into a compressed representation based on the weights (defined in Eq.~\ref{eq:weight}) of the tokens within the group. Specifically, the weight of the semantic center is its cosine similarity with $\bar{q}$, and the weight of the non-center token is the product of their cosine similarity with $\bar{q}$ and the center of their respective group.}
\label{fig:SeCo}
\end{figure}

In this section, we present SeCo, a \emph{semantic-driven} framework that compresses long contexts into compact soft tokens while preserving task-relevant information. SeCo comprises an encoder for feature extraction and a decoder for downstream generation. The core compression mechanism proceeds in three stages: \textit{Semantic Center Selection}, \textit{Semantic Assignment}, and \textit{Consistency Merging}. By decoupling compression process from physical token layout, SeCo mitigates position bias and ensures semantic consistency within each compressed tokens.

\subsection{Semantic Feature Extraction}
Given an input sequence $X$ composed of a context $X^{\text{context}}$ and a query $X^{\text{query}}$, we employ a pre-trained language model as the encoder to extract contextualized hidden states:
\begin{equation}
    H = \texttt{Encoder}(X),
\end{equation}
where $H \in \mathbb{R}^{N \times d}$ denotes the last-layer hidden states, with $N$ being the total sequence length and $d$ the hidden dimension. We partition $H$ into context and query components, $H^{\text{context}} \in \mathbb{R}^{N_c \times d}$ and $H^{\text{query}} \in \mathbb{R}^{N_q \times d}$, where $N_c$ and $N_q$ represent their respective sequence lengths.

\subsection{Semantic Center Selection}
To achieve a target compression rate $\tau$, we determine the number of compressed tokens $K$ as:
\begin{equation}
\label{eq:topK}
    K = \max\!\left(2,\, \left\lceil \frac{N_c}{\tau} \right\rceil\right).
\end{equation}
Guided by the principle of task relevance, we first aggregate the query information into a single representation $\bar{q}$ by averaging its token embeddings:
\begin{equation}
    \bar{q} = \frac{1}{N_q} \sum_{i=1}^{N_q} h^{\text{query}}_i,
\end{equation}
where $h^{\text{query}}_i \in \mathbb{R}^d$ is the $i$-th token embedding in $H^{\text{query}}$.

Next, we quantify the semantic relevance of each context token to the query via cosine similarity:
\begin{equation}
    r_i = \frac{h^{\text{context}}_i \cdot \bar{q}}{\|h^{\text{context}}_i\| \|\bar{q}\|} \, .
    \label{eq:relevance}
\end{equation}
We select the $K$ context tokens with the highest similarity scores as semantic centers:
\begin{equation}
    \mathcal{C} = \left\{ i \in \{1,\dots,N_c\} \;\middle|\; r_i \in \operatorname{\textit{top}\text{-}\textit{k}}_{j=1}^{N_c}\{r_j\} \right\}, \quad |\mathcal{C}| = K.
\end{equation}
These centers serve as semantic anchors, ensuring that the compressed representation retains the most task-critical information.

\subsection{Semantic Assignment and Consistency Merging}

\paragraph{Semantic Assignment.}
Remaining non-center tokens, denoted as $H^{\text{non}} = H^{\text{context}} \setminus \{h_j\}_{j \in \mathcal{C}}$, we first compute the cosine similarity $s_{i,k}$ between $h^{\text{non}}_i$ and every center $h^{\text{center}}_k$:
\begin{equation}
    s_{i,k} = \frac{h^{\text{non}}_i \cdot h^{\text{center}}_k}{\|h^{\text{non}}_i\| \|h^{\text{center}}_k\|} \, .
    \label{eq:similarity}
\end{equation}
Then, we compute a joint assignment score $v(i, k)$:
\begin{equation}
    v(i, k) = r_i \cdot s_{i,k} \, .
    \label{eq:value}
\end{equation}
Token $i$ is assigned to the cluster $k^*$ that maximizes this score: $k^* = \arg\max_k\, v(i, k)$. This mechanism ensures that tokens are grouped not merely by proximity, but by their combined semantic coherence and utility to the query.

\paragraph{Consistency Merging.}

Within each compression group $\mathcal{G}_k$ (comprising the center and its assigned members), we first define a consistency weight $w_i$ for each token:
\begin{equation}
    w_i = \begin{cases}
        r_i & \text{if } i \in \mathcal{C}, \\
        v(i,k^*) & \text{otherwise.}
    \end{cases}
    \label{eq:weight}
\end{equation}
Center tokens are weighted by their direct task relevance ($r_i$), while non-center tokens are weighted by their joint assignment score ($v(i,k^*)$), balancing local semantic consistency with global task alignment. 

We then generate a  compressed representation $\tilde{h}_k$ via a consistency-weighted aggregation:
\begin{equation}
    \alpha_i = \frac{\exp(w_i)}{\sum_{j \in \mathcal{G}_k} \exp(w_j)}, \quad \tilde{h}_k = \sum_{i \in \mathcal{G}_k} \alpha_i \cdot h_i.
\end{equation} 

The final compressed context is formed by concatenating these aggregated representations:
\begin{equation}
    \widetilde{H} = \left[\tilde{h}_1,\, \ldots,\, \tilde{h}_K\right] \in \mathbb{R}^{K \times d}.
\end{equation}
$\widetilde{H}$ serves as the compressed embeddings for the subsequent training/decoding stage.

\subsection{Training Objective}
The training loss is:
\begin{equation}
    \mathcal{L}_{\text{nll}} = -\sum_{i=1}^{N_{a}} \log p_{\phi}\!\left(a_i \mid \widetilde{H}, X^{\text{query}},\, a_{<i}\right) \, ,
\end{equation}
where $p_\phi(\cdot)$ is the decoder's output distribution; $a_i$ denotes the $i$-th answer token; $N_{a}$ represents the token length of answer. Both the encoder and decoder are fine-tuned using LoRA.
\definecolor{lightblue}{rgb}{0.92, 0.95, 1.0}
\section{Experiment}
\label{4-experiment}

In this section, we aim to answer the following Research Questions (RQs): (1) How does SeCo perform on various downstream tasks (RQ1)?  (2) Compared with other methods, does SeCo have an advantage in inference latency (RQ2)? (3) How effective are the individual components of SeCo (RQ3)?

\subsection{Settings}

\paragraph{Training.} 
The training paradigm of SeCo is shown in Figure~\ref{fig:FT_Training}. Our method uses the MRQA~\citep{DBLP:conf/acl-mrqa/FischTJSCC19} datasets, partitioned into in-domain and out-of-domain sets. The in-domain data, used for both training and internal testing, comprises six subsets: SQuAD~\citep{DBLP:conf/emnlp/RajpurkarZLL16}, NewsQA~\citep{DBLP:conf/rep4nlp/TrischlerWYHSBS17}, TriviaQA~\citep{DBLP:conf/acl/JoshiCWZ17}, SearchQA~\citep{DBLP:journals/corr/DunnSHGCC17}, HotpotQA~\citep{DBLP:conf/emnlp/Yang0ZBCSM18} and NaturalQuestions~\citep{DBLP:journals/tacl/KwiatkowskiPRCP19}. To evaluate generalization, we employed an out-of-domain test set consisting of DROP~\citep{DBLP:conf/naacl/DuaWDSS019}, BioASQ~\citep{DBLP:journals/bmcbi/TsatsaronisBMPZ15}, DuoRC.ParaphraseRC~\citep{DBLP:conf/acl/KhapraSSA18}, TextbookQA~\citep{DBLP:conf/cvpr/KembhaviSSCFH17}, RelationExtraction~\citep{DBLP:conf/conll/LevySCZ17} and RACE~\citep{DBLP:conf/emnlp/LaiXLYH17}.

During training, we set two compression rates for each training sample (i.e., 16$\times$ and 32$\times$). SeCo is implemented based on \texttt{LLaMA-3.2-1B-Instruct} and \texttt{Qwen3-4B-Instruct}. Furthermore, to assess SeCo's performance in summarization and long-context QA, we incorporate DialogSum~\citep{DBLP:conf/acl/ChenLCZ21}, a multi-turn dialogue summarization benchmark, and NarrativeQA~\citep{DBLP:journals/tacl/KociskySBDHMG18}, which demands deep comprehension and reasoning over long-form stories and scripts. Further dataset descriptions and training details are provided in Appendix~\ref{apx:dataset_details} and Appendix~\ref{apx:train_details}, respectively. 

\paragraph{Evaluation Metrics.}
We evaluate QA performance using Exact Match (EM)~\citep{DBLP:journals/corr/abs-2510-08907} and F1 score~\citep{lin-2004-rouge}, and assess summarization quality via BERT-F1~\citep{DBLP:conf/iclr/ZhangKWWA20} computed with \texttt{roberta-large}\footnote{\url{https://huggingface.co/FacebookAI/roberta-large}}.

\paragraph{Baselines.}
We conduct a comparison with encoder-decoder based context compression methods, including task-agnostic methods (i.e., ICAE~\citep{DBLP:conf/iclr/00010WWCW24}, EPL~\citep{DBLP:conf/emnlp/ZhaoLLHXXZ25}, SAC~\citep{DBLP:journals/corr/abs-2510-08907}, 500xCompressor~\citep{DBLP:conf/acl/LiSC25}) and task-aware methods (i.e., RAM~\citep{DBLP:journals/corr/abs-2602-01840}, COMI~\citep{DBLP:journals/corr/abs-2602-01719}). For methods employing specific position layout strategies, we append the suffixes (EPL) or (DPL) to denote the adopted scheme. Specifically, Enhanced Position Layout (EPL) reassigns position IDs to learnable or compressed tokens to optimize structural alignment. In contrast, Default Position Layout (DPL) retains the standard sequential positional indices for these tokens without additional reorganization. \textbf{All baselines use the same training datasets and hyperparameters.}

\subsection{Main Result}
For RQ1, we analyze the performance of SeCo via its performance on diverse downstream tasks. Tables~\ref{tab:res_iid} and~\ref{tab:res_ood} report the results on in-domain and out-of-domain test sets. Under compression constraints of 16$\times$ and 32$\times$, SeCo demonstrates superior performance across various backbone models (i.e., \texttt{LLaMA3.2-1B-Instruct} and \texttt{Qwen3-4B-Instruct}). SeCo consistently outperforms both \textit{task-agnostic} and \textit{task-aware} baselines, across most test datasets in terms of average EM and F1. This highlights SeCo's ability to aggregate semantic space knowledge based on semantic cues provided by queries while eliminating the need of \textit{positional priors} used in the baselines. Figure~\ref{fig:summary_comparison} presents the results of the long-context QA and summarization. SeCo exhibits a consistent performance advantage over task-aware baselines (i.e., RAM and COMI). Furthermore, we conduct comprehensive stress tests on SeCo across varying compression rates, which demonstrate that SeCo consistently maintains superior performance in both in-domain and out-of-domain evaluations, underscoring its robustness (see Appendix~\ref{apx:stress_test}).

\begin{table}[htbp]
\centering
\tiny
\setlength{\tabcolsep}{3.0pt}
\caption{Main results of benchmark tests on In-Domain (ID) datasets. We \textbf{bold} the optimal results.}
\label{tab:res_iid}
\resizebox{\textwidth}{!}{%
\begin{tabular}{l|cccccccccccc|cc}
\toprule
\multirow{2}{*}{\textbf{Methods}} & \multicolumn{2}{c}{\textbf{SQuAD}} & \multicolumn{2}{c}{\textbf{NewsQA}} & \multicolumn{2}{c}{\textbf{TriviaQA}} & \multicolumn{2}{c}{\textbf{SearchQA}} & \multicolumn{2}{c}{\textbf{HotpotQA}} & \multicolumn{2}{c}{\textbf{NQ}} & \multicolumn{2}{c}{\textbf{Average}} \\
\cmidrule(lr){2-3} \cmidrule(lr){4-5} \cmidrule(lr){6-7} \cmidrule(lr){8-9} \cmidrule(lr){10-11} \cmidrule(lr){12-13} \cmidrule(lr){14-15}
& \textbf{EM} & \textbf{F1} & \textbf{EM} & \textbf{F1} & \textbf{EM} & \textbf{F1} & \textbf{EM} & \textbf{F1} & \textbf{EM} & \textbf{F1} & \textbf{EM} & \textbf{F1} & \textbf{EM} & \textbf{F1}  \\

\midrule
\multicolumn{15}{c}{\textit{16x compression constraint}} \\
\midrule
\multicolumn{15}{c}{\textbf{LLaMA-3.2-1B-Instruct}} \\
\midrule

RAM & 34.12 & 51.49 & 19.90 & 35.18 & 41.07 & 50.23 & 51.47 & 61.69 & 27.94 & 47.25 & 27.49 & 43.97 & 36.99 & 51.03 \\
COMI & 27.26 & 42.42 & 14.62 & 27.13 & 50.85 & 57.79 & 62.88 & 70.65 & 33.67 & 51.28 & 34.21 & 50.88 & 42.07 & 54.37 \\
500x (DPL) & 14.92 & 26.96 & 10.02 & 23.06 & 39.99 & 49.73 & 47.34 & 61.00 & 26.28 & 41.97 & 25.64 & 41.88 & 30.89 & 44.46 \\
500x (EPL) & 17.98 & 32.42 & 11.35 & 26.51 & 42.36 & 52.50 & 48.49 & 62.25 & 29.84 & 47.19 & 26.39 & 44.16 & 32.72 & 47.46 \\
SAC (EPL) & 19.00 & 33.96 & 11.89 & 27.93 & 41.01 & 51.57 & 47.66 & 61.47 & 30.35 & 46.98 & 26.38 & 43.76 & 32.56 & 47.38 \\
ICAE (EPL) & 17.94 & 32.09 & 10.59 & 24.37 & 36.06 & 45.71 & 44.17 & 57.67 & 26.98 & 42.56 & 24.70 & 42.82 & 29.89 & 44.24 \\
\midrule
\rowcolor{lightblue}
\textbf{SeCo} & \textbf{51.06} & \textbf{70.71 }& \textbf{29.91} & \textbf{50.81} & \textbf{63.47} & \textbf{71.94} & \textbf{68.60} & \textbf{76.66} & \textbf{46.28} & \textbf{67.09} & \textbf{40.64} & \textbf{60.51} & \textbf{53.52} & \textbf{68.55} \\

\midrule

\multicolumn{15}{c}{\textbf{Qwen3-4B-Instruct}} \\
\midrule

RAM & 17.28 & 29.25 & 7.48 & 16.87 & 40.35 & 45.97 & 49.83 & 56.26 & 18.86 & 31.75 & 15.70 & 29.40 & 28.96 & 38.75 \\
COMI & 15.88 & 27.85 & 6.77 & 16.31 & 39.87 & 45.67 & 49.27 & 55.96 & 17.69 & 31.08 & 16.63 & 29.95 & 28.52 & 38.39 \\
500x (DPL) & 23.93 & 39.96 & 14.58 & 31.06 & 50.96 & 60.73 & 55.91 & 69.28 & 35.74 & 52.68 & 35.74 & 54.02 & 39.99 & 55.03 \\
500x (EPL) & 32.50 & 52.67 & 20.16 & 39.70 & 55.00 & 65.15 & 59.45 & 71.93 & \textbf{44.62} & \textbf{63.13} & \textbf{42.09} & \textbf{61.55} & 45.82 & 62.04 \\
SAC (EPL) & 33.77 & 53.63 & 20.25 & 41.08 & 54.28 & 64.78 & 56.50 & 69.66 & 44.06 & 62.16 & 42.05 & 60.64 & 45.03 & 61.30 \\
ICAE (EPL) & 32.21 & 50.40 & 15.08 & 32.20 & 49.36 & 59.30 & 53.56 & 66.95 & 37.10 & 53.80 & 39.51 & 58.35 & 41.59 & 57.20 \\
\midrule
\rowcolor{lightblue}
\textbf{SeCo} & \textbf{44.98} & \textbf{65.56} & \textbf{27.40} & \textbf{48.63} & \textbf{60.94} & \textbf{70.19} & \textbf{69.77} & \textbf{78.22} & 38.30 & 61.16 & 32.54 & 54.84 & \textbf{49.65} & \textbf{65.84} \\
\midrule
\multicolumn{15}{c}{\textit{32x compression constraint}} \\
\midrule
\multicolumn{15}{c}{\textbf{LLaMA-3.2-1B-Instruct}} \\
\midrule

RAM & 28.81 & 43.75 & 17.57 & 30.33 & 39.38 & 47.96 & 46.80 & 56.83 & 23.91 & 40.98 & 25.05 & 41.23 & 33.33 & 46.32 \\
COMI & 18.99 & 32.70 & 12.89 & 24.73 & 44.38 & 51.64 & 59.05 & 66.69 & 26.76 & 42.76 & 29.10 & 45.99 & 36.64 & 48.52 \\
500x (DPL) & 12.51 & 23.54 & 9.43 & 21.82 & 36.69 & 46.16 & 46.63 & 59.93 & 22.44 & 36.39 & 23.09 & 38.67 & 28.81 & 41.69 \\
500x (EPL) & 13.98 & 26.15 & 9.76 & 22.21 & 37.93 & 48.15 & 44.75 & 58.16 & 23.49 & 38.21 & 23.28 & 39.65 & 28.86 & 42.34 \\
SAC (EPL) & 14.75 & 27.13 & 9.26 & 22.74 & 38.79 & 48.52 & 46.73 & 60.56 & 23.22 & 37.93 & 23.96 & 40.77 & 29.78 & 43.52 \\
ICAE (EPL) & 15.44 & 28.30 & 8.86 & 20.65 & 34.80 & 44.93 & 44.13 & 58.21 & 24.17 & 38.30 & 23.86 & 40.34 & 28.66 & 42.36 \\
\midrule
\rowcolor{lightblue}
\textbf{SeCo} & \textbf{50.34} & \textbf{69.47} & \textbf{29.91} & \textbf{49.91} & \textbf{60.46} & \textbf{69.51} & \textbf{68.00} & \textbf{75.67} & \textbf{45.18} & \textbf{65.43} & \textbf{39.02} & \textbf{58.85} & \textbf{52.35} & \textbf{67.12} \\

\midrule

\multicolumn{15}{c}{\textbf{Qwen3-4B-Instruct}} \\
\midrule

RAM & 13.45 & 24.10 & 5.25 & 13.01 & 37.19 & 42.56 & 46.13 & 52.36 & 14.01 & 25.13 & 13.38 & 25.51 & 25.60 & 34.42 \\
COMI & 13.11 & 24.11 & 4.77 & 12.51 & 37.55 & 42.77 & 46.50 & 52.61 & 13.83 & 25.12 & 14.16 & 26.46 & 25.82 & 34.70 \\
500x (DPL) & 23.34 & 38.30 & 14.34 & 30.50 & 50.89 & 60.38 & 56.65 & 71.69 & 34.43 & 51.05 & 34.22 & 50.97 & 39.61 & 54.52 \\
500x (EPL) & 23.69 & 40.29 & 14.84 & 32.31 & 50.98 & 61.01 & 58.28 & 70.56 & 37.54 & 54.89 & \textbf{36.55} & 37.54 & 41.03 & 56.05 \\
SAC (EPL) & 25.02 & 41.64 & 15.86 & 33.88 & 53.50 & 63.32 & 56.41 & 69.45 & 36.84 & 53.93 & 35.46 & 54.03 & 40.82 & 56.07 \\
ICAE (EPL) & 25.96 & 43.78 & 12.51 & 27.81 & 46.68 & 56.73 & 46.40 & 60.19 & 32.20 & 48.20 & 33.42 & 52.06 & 36.00 & 51.41 \\
\midrule
\rowcolor{lightblue}
\textbf{SeCo} & \textbf{42.96} & \textbf{63.46} & \textbf{26.23} &\textbf{ 46.57} & \textbf{60.71} & \textbf{69.43} & \textbf{67.29} & \textbf{75.59} & \textbf{38.49} & \textbf{60.37} & 35.23 & \textbf{56.26} & \textbf{49.06} & \textbf{64.67} \\

\midrule
\end{tabular}%
}
\end{table}

\begin{table}[htbp]
\centering
\tiny
\setlength{\tabcolsep}{3.0pt}
\caption{Main results of benchmark tests on Out-Of-Domain (OOD) datasets using \texttt{LLaMA-3.2-1B-Instruct} as backbone.}
\label{tab:res_ood}
\resizebox{\textwidth}{!}{%
\begin{tabular}{l|cccccccccccc|cc}
\toprule
\multirow{2}{*}{\textbf{Methods}} & \multicolumn{2}{c}{\textbf{DROP}} & \multicolumn{2}{c}{\textbf{BioASQ}} & \multicolumn{2}{c}{\textbf{DuoRC}} & \multicolumn{2}{c}{\textbf{TQA}} & \multicolumn{2}{c}{\textbf{RE}} & \multicolumn{2}{c}{\textbf{RACE}} & \multicolumn{2}{c}{\textbf{Average}} \\
\cmidrule(lr){2-3} \cmidrule(lr){4-5} \cmidrule(lr){6-7} \cmidrule(lr){8-9} \cmidrule(lr){10-11} \cmidrule(lr){12-13} \cmidrule(lr){14-15}
& \textbf{EM} & \textbf{F1} & \textbf{EM} & \textbf{F1} & \textbf{EM} & \textbf{F1} & \textbf{EM} & \textbf{F1} & \textbf{EM} & \textbf{F1} & \textbf{EM} & \textbf{F1} & \textbf{EM} & \textbf{F1}  \\
\midrule
\multicolumn{15}{c}{\textit{16x compression constraint}} \\
\midrule

RAM & 11.11 & 18.39 & 22.67 & 34.72 & 19.79 & 30.47 & 22.09 & 29.26 & 32.94 & 46.93 & 13.50 & 25.33 & 22.83 & 33.74 \\
COMI & 16.10 & 24.67 & 27.93 & 40.08 & 9.86 & 16.42 & 15.17 & 22.55 & 42.57 & 58.81 & 9.35 & 19.37 & 24.46 & 35.54 \\
500x (DPL) & 21.42 & 29.94 & 24.34 & 35.96 & 6.80 & 12.89 & 17.43 & 30.78 & 21.64 & 34.15 & 2.52 & 14.16 & 17.72 & 28.54 \\
500x (EPL) & 23.02 & 32.85 & 20.28 & 34.05 & 15.12 & 23.66 & 24.75 & 41.10 & 30.66 & 43.08 & 3.41 & 18.90 & 22.60 & 35.05 \\
SAC (EPL) & 24.55 & 33.89 & 20.41 & 31.60 & 15.19 & 24.34 & 23.35 & 39.34 & 33.96 & 47.04 & 3.71 & 18.78 & 23.68 & 35.86 \\
ICAE (EPL) & 21.82 & 31.55 & 26.53 & 38.18 & 11.53 & 18.38 & 18.50 & 30.76 & 38.74 & 54.18 & 2.08 & 14.54 & 24.23 & 36.14 \\
\midrule
\rowcolor{lightblue}
\textbf{SeCo} & \textbf{26.61} & \textbf{39.94}& \textbf{45.68} & \textbf{60.07} & \textbf{35.04} & \textbf{48.76} & \textbf{42.12} & \textbf{53.48} & \textbf{56.82} & \textbf{73.29} & \textbf{23.29} & \textbf{39.98} & \textbf{42.33} & \textbf{56.78} \\

\midrule

\multicolumn{15}{c}{\textit{32x compression constraint}} \\
\midrule

RAM & 16.50 & 23.67 & 25.53 & 32.91 & 16.99 & 25.95 & 20.63 & 28.37 & 24.86 & 37.15 & 12.91 & 23.60 & 20.94 & 30.32 \\
COMI & 13.91 & 21.38 & 24.60 & 36.59 & 8.59 & 13.72 & 16.50 & 24.92 & 36.43 & 51.46 & 7.27 & 16.56 & 21.58 & 31.98 \\
500x (DPL) & 19.96 & 28.35 & 22.47 & 32.47 & 6.26 & 11.54 & 16.43 & 28.16 & 28.26 & 39.96 & 1.48 & 11.96 & 18.91 & 28.75 \\
500x (EPL) & 20.96 & 29.43 & 20.01 & 30.79 & 10.73 & 18.73 & 21.29 & 35.80 & 31.07 & 42.59 & 2.97 & 15.49 & 21.10 & 32.02 \\
SAC (EPL) & 20.83 & 30.42 & 19.95 & 31.32 & 10.13 & 17.27 & 21.62 & 37.18 & 25.37 & 37.15 & 3.12 & 15.75 & 19.30 & 30.60 \\
ICAE (EPL) & 21.22 & 30.45 & 23.01 & 34.70 & 3.80 & 7.31 & 19.49 & 32.57 & 36.23 & 49.57 & 2.97 & 14.81 & 21.83 & 32.60 \\
\midrule
\rowcolor{lightblue}
\textbf{SeCo} & \textbf{25.95} & \textbf{39.07} & \textbf{43.35} & \textbf{58.47} & \textbf{34.71} & \textbf{48.53} & \textbf{41.78} & \textbf{52.67} & \textbf{57.02} & \textbf{73.54} & \textbf{23.00} & \textbf{40.04} & \textbf{41.80} & \textbf{56.31} \\
\midrule

\end{tabular}%
}
\end{table}

\begin{figure}[htbp]
    \centering
    \begin{subfigure}[b]{0.46\textwidth}
        \centering
        \hspace*{-0.5cm}
        \includegraphics[width=\linewidth]{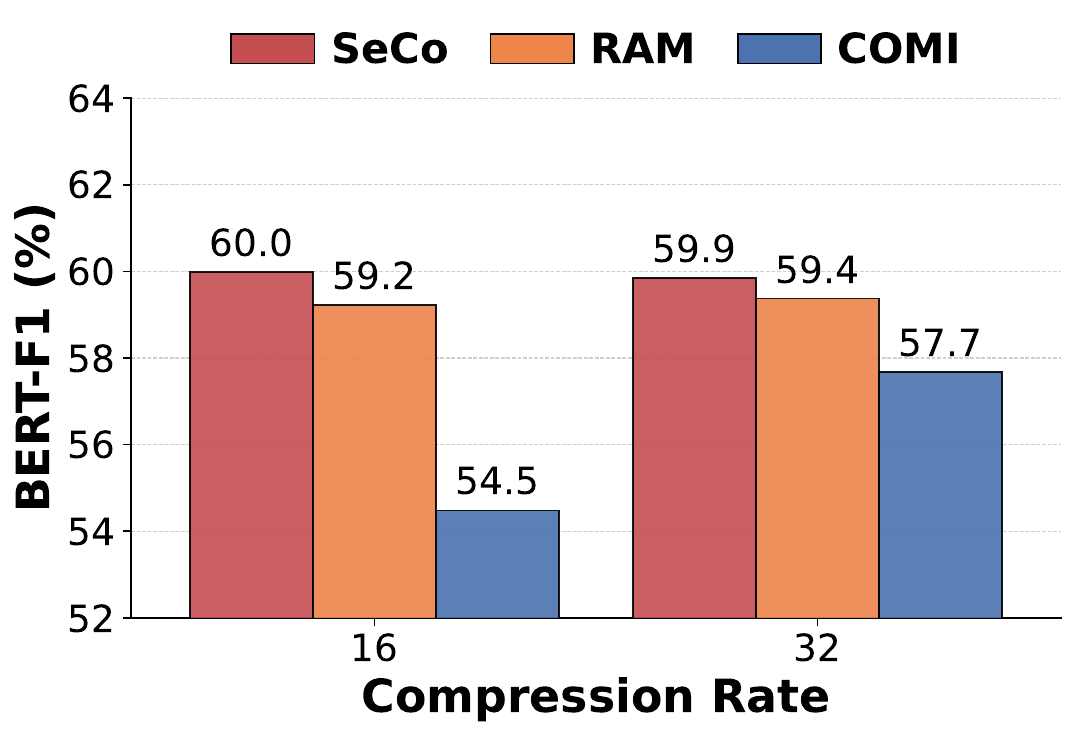}
        \caption{DialogSum}
        \label{fig:dialogsum}
    \end{subfigure}
    \hfill
    \begin{subfigure}[b]{0.46\textwidth}
        \centering
        \hspace*{-0.5cm}
        \includegraphics[width=\linewidth]{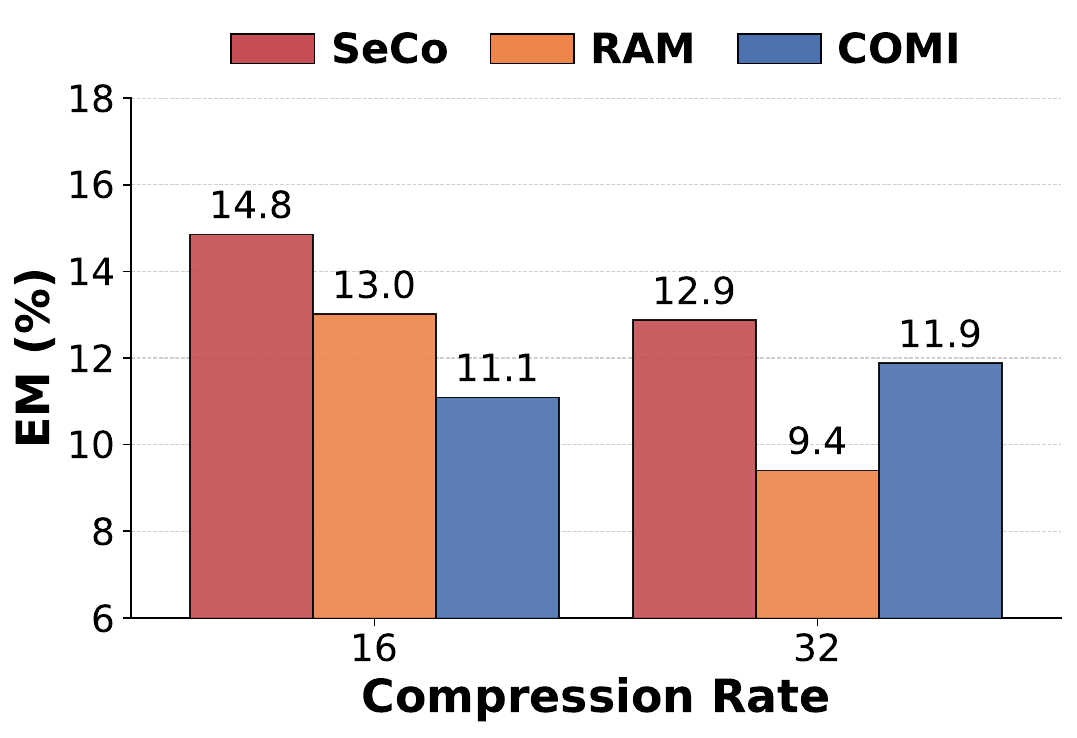}
        \caption{NarrativeQA}
        \label{fig:narrativeqa}
    \end{subfigure}
    \caption{Additional results on long context and summarization benchmarks compared with task-aware methods (i.e., COMI and RAM) using \texttt{Qwen3-4B-Instruct} as the backbone.}
    \label{fig:summary_comparison}
\end{figure}

\subsection{Efficiency Analysis}
For RQ2, we conduct the following inference latency experiment. SeCo replaces the original context with soft prompt to accelerate the inference process, thereby reducing the inference cost of processing the original context in the generation stage by a factor of $\tau$. The overall computational overhead includes the computational costs of both the compression and generation stages.

\begin{table}[h]
\centering
\footnotesize
\setlength{\tabcolsep}{2.0pt}
\caption{Latency evaluation (in milliseconds) under a 16$\times$ compression constraint across different datasets. Compression latency and generation latency denote the latency for the compression and generation phases, respectively. Total latency is the sum of both phases. The best results are highlighted in \textbf{bold}, and the second-best results are \underline{underlined}.}
\label{tab:eval_lat}
\begin{tabular}{l|ccc|ccc|ccc}
\toprule
\multirow{2}{*}{\textbf{Methods}} &\multicolumn{3}{c|}{\textbf{SQuAD}} & \multicolumn{3}{c|}{\textbf{NewsQA}} & \multicolumn{3}{c}{\textbf{NQ}}  \\
\cmidrule(lr){2-4} \cmidrule(lr){5-7}\cmidrule(lr){8-10}
& \textbf{Comp. Lat.} & \textbf{Gen. Lat.} & \textbf{Total Lat.} & \textbf{Comp. Lat.} & \textbf{Gen. Lat.} & \textbf{Total Lat.} & \textbf{Comp. Lat.} & \textbf{Gen. Lat.} & \textbf{Total Lat.} \\

\midrule
COMI & 14.10  & \textbf{44.93}  & 59.03  & 28.11 & \underline{46.80} & 74.91 &  17.81 &  \underline{49.24} &  \underline{67.05} \\
500x (DPL) & 11.40  & 47.54 & 58.94 & 21.29 & 50.46 & 71.75 & 14.45  & 91.86  & 106.31 \\
500x (EPL) & \underline{11.22}  & 48.30 &  59.52 & 21.22 & 48.74 & 69.96 & \underline{14.25} & 57.52 & 71.77 \\
SAC (EPL) & 12.31 & 51.32 & 63.63 & 22.59  &  56.58  & 79.17 & 15.10  & 66.09  &  81.19 \\ 
ICAE (EPL) & 11.56 & 45.90  &  \underline{57.46}  & \underline{21.16} & 48.47 & \underline{69.63} &  14.56   &  61.24   &  75.80 \\
\midrule
\rowcolor{lightblue}
\textbf{SeCo} & \textbf{10.23} & \underline{45.67} & \textbf{55.90} & \textbf{14.30} & \textbf{46.43} & \textbf{60.73} & \textbf{11.90}  &  \textbf{46.77}  & \textbf{58.67} \\
\bottomrule
\end{tabular}
\end{table}
The compression process of SeCo comprises three distinct phases:
(1)~a single full forward pass of the encoder over the input sequence;
(2)~compute the cosine similarity between each of the context tokens and the query semantics $\bar{q}$;
(3)~compute the pairwise similarity matrix between the non-center tokens and the semantic centers, followed by consistency-weighted  aggregation.

The compression process can be represented as:
\begin{equation}
\mathrm{FLOPs}^{\text{comp.}} = F^{\mathrm{Encoder}}(L) + F^{\mathrm{Selection}}(L_c, d) + F^{\mathrm{Assignment}}(L_c, K, d) + F^{\mathrm{Merging}}(L_c, K, d),
\end{equation}
where $L_c$ and $d$ represent the context lengths and hidden dimension, respectively. $K$ is the number of semantic centers defined in Eq.~\ref{eq:topK}.
Specifically, $F^{\mathrm{Encoder}}(L)$ is the cost of prefilling process;
$F^{\mathrm{Selection}}(L_c,d)$ denotes the \textit{Semantic Center Selection} cost
(complexity~$O(L_c d)$, with $L_c \ll L^2$);
$F^{\mathrm{Assignment}}(L_c,K,d)$ refers the \textit{Semantic Assignment} cost
(complexity~$O(L_c Kd)$, with $L_cK \ll L^2$);
and $F^{\mathrm{Merging}}(L_c,K,d)$ represents the \textit{Consistency Merging} cost
(complexity~$O(L_c d)$, with $L_c \ll L^2$). Therefore, aforementioned operations introduce only lightweight overhead.

For the generation stage, assuming the answer length is $L_a$, the decoder performs $L_a$ forward passes autoregressively over the compressed prefix of length $K + L_q$. The FLOPs for the $i$-th forward pass are:
\begin{equation}
    \mathrm{FLOPs}^{\text{generation}} = \sum_{i=1}^{L_a}F^{\mathrm{Decoder}}(K + L_q,\ i),
\end{equation}
where $F^{\mathrm{Decoder}}(n, i)$ denotes the FLOPs of a single decoder forward step with prefix length $n$ at generation step $i$. 
SeCo achieves the lowest total latency among all baselines, delivering superior speed under a 16$\times$ compression rate across multiple test datasets (see Table~\ref{tab:eval_lat}).

\subsection{Ablation Study}
To investigate RQ3 and understand the specific contributions of each component within SeCo, we conduct the following three ablation experiments. (see Table~\ref{tab:ablation}): (1) Our \textit{w/o} Query indicates that during the assignment and merging process, we do not use the cosine similarity between the non-center tokens and $\bar{q}$ as a weight reference, but instead rely solely on the cosine similarity between the non-center tokens and the semantic centers; (2) Our \textit{w/o} Consist. Merging refers to directly select the \textit{top}\text{-}\textit{k} most relevant tokens based on their similarity scores with $\bar{q}$ to achieve the target compression rate, bypassing the entire assignment and merging process; (3) Our \textit{w/} Uniform Sample means that the semantic centers are selected through equidistant and uniform sampling across the sequence, while keeping the subsequent assignment and merging logic identical to the original algorithm.

Omitting any component leads to a substantial performance drop, validating the necessity of each module. Specifically, removing query guidance (\textit{w/o} Query) eliminates relevance weighting between non-center tokens and the query, resulting in a lack of \textit{task-specific} focus during semantic assignment and consistency merging. Disabling the assignment mechanism (\textit{w/o} Consist. Merging) restricts the model to hard filtering, thereby discarding valuable contextual information from unselected tokens. Finally, replacing semantic center selection with uniform sampling (\textit{w/} Uniform Sample) ignores input semantic distributions, potentially aligning centers with query-irrelevant regions and degrading representation accuracy.
\begin{table}[htbp]
\centering
\caption{Ablation study on diverse benchmarks under 16$\times$ compression constraint. We use \texttt{LLaMA-3.2-1B-Instruct} as the backbone. ID and OOD refer to in-domain and out-of-domain  datasets, respectively.}
\label{tab:ablation}
\footnotesize
\setlength{\tabcolsep}{2.5pt}
\renewcommand{\arraystretch}{1.1}
\begin{tabular}{l | cc cc cc | cc cc cc | cc}
\toprule
\multirow{3}{*}{\textbf{Methods}} & \multicolumn{6}{c|}{\textbf{ID}} & \multicolumn{6}{c|}{\textbf{OOD}} & \multicolumn{2}{c}{\textbf{Overall}} \\
\cmidrule(lr){2-7} \cmidrule(lr){8-13} \cmidrule(lr){14-15}
& \multicolumn{2}{c}{\textbf{SearchQA}} & \multicolumn{2}{c}{\textbf{HotpotQA}} & \multicolumn{2}{c|}{\textbf{NQ}}  
& \multicolumn{2}{c}{\textbf{BioASQ}} & \multicolumn{2}{c}{\textbf{TQA}} & \multicolumn{2}{c|}{\textbf{RE}} & 
\multicolumn{2}{c}{\textbf{Average}} \\

\cmidrule(lr){2-3} \cmidrule(lr){4-5} \cmidrule(lr){6-7} \cmidrule(lr){8-9} \cmidrule(lr){10-11} \cmidrule(lr){12-13} \cmidrule(lr){14-15}
& \textbf{EM} & \textbf{F1} & \textbf{EM} & \textbf{F1} & \textbf{EM} & \textbf{F1} & \textbf{EM} & \textbf{F1} & \textbf{EM} & \textbf{F1} & \textbf{EM} & \textbf{F1} & \textbf{EM} & \textbf{F1} \\
\midrule
\rowcolor{lightblue}
\textbf{Default}     & \textbf{68.60} & \textbf{76.66}  & \textbf{46.28} & \textbf{67.09}  & \textbf{40.64} & \textbf{60.51} & \textbf{45.68} & \textbf{60.07}  & \textbf{42.12} & \textbf{53.48} & \textbf{56.82}  & \textbf{73.29} & \textbf{54.21} & \textbf{68.66} \\
\midrule
\textit{w/o} Query   & 57.65 & 67.62  & 41.98 & 63.70  & 35.58 & 55.25 & 31.32 & 47.73  & 29.87 & 39.70 & 52.04 & 69.30 & 46.28 & 61.65 \\
\textit{w/o} Consist. Merging   & 57.37 & 67.30   & 41.64 & 63.33  & 35.55 & 55.72 & 30.78 & 46.47  & 30.74 & 40.53 & 55.22 & 71.15 & 46.35 & 61.73 \\
\textit{w/} Uniform Sample & 61.01 & 70.62  & 40.96 & 62.97 & 35.78  & 55.64  & 29.59 & 45.20  & 32.87 & 42.81 & 51.63 & 68.15 & 47.59 & 62.83 \\
\midrule
\end{tabular}
\end{table}

\section{Conclusion}
We identify position bias as a bottleneck in existing soft prompt compression: compression effectiveness is often constraint by position prior, which can lead positional encoding bias and group-based bias. To overcome this limitation, we propose SeCo, a query-aware semantic-space compression framework that selects task-relevant centers and aggregates remaining context tokens via consistency merging while eliminating the need of positional priors. Across extensive benchmarks, backbones, and compression rates, SeCo consistently delivers superior downstream performance and better efficiency than prior prompt compression baselines. These results suggest that effective long-context compression can be learned \emph{directly} in semantic space. We hope SeCo offers a simple and promising step toward more effective and efficient long-context modeling.

\bibliographystyle{plainnat}
\bibliography{refs}

\medskip

\newpage
\appendix
\section{Positional Encoding Bias Statement}
\label{apx:Position Encoding Bias}
\begin{figure}[htbp]
    \centering
    
    \includegraphics[width=1.0\textwidth]{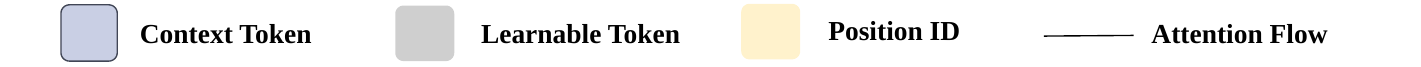}
    \vspace{0.5em}

    \begin{subfigure}[b]{0.45\textwidth}
        \includegraphics[width=\textwidth]{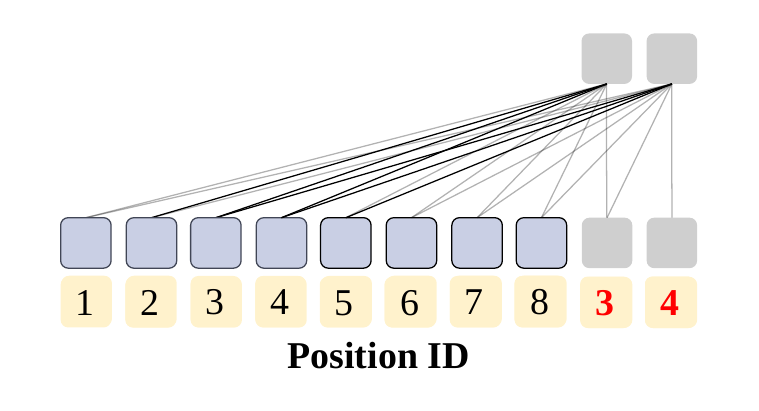}
        \caption{$\beta = 0.25$}
    \end{subfigure}
    \hfill
    \begin{subfigure}[b]{0.45\textwidth}
        \includegraphics[width=\textwidth]{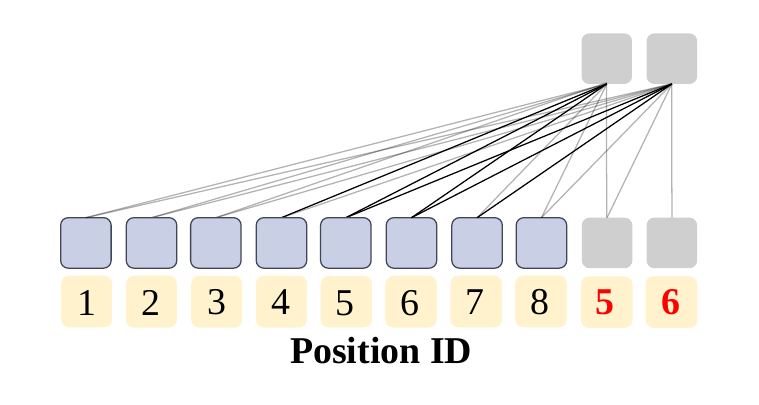}
        \caption{$\beta = 0.50$}
    \end{subfigure}

    \vspace{0.5em}

    \begin{subfigure}[b]{0.45\textwidth}
        \includegraphics[width=\textwidth]{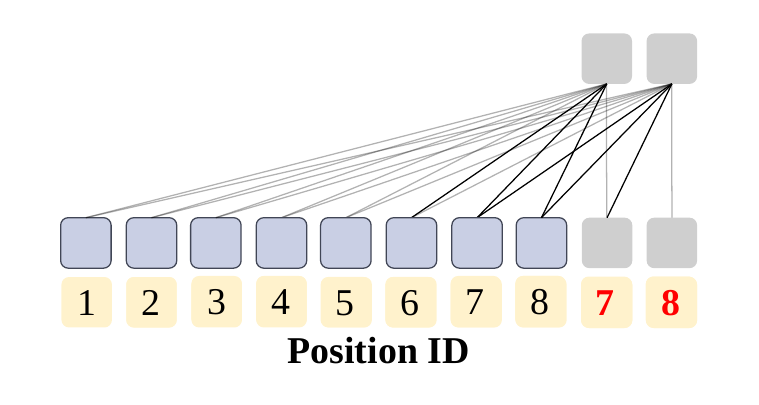}
        \caption{$\beta = 0.75$}
    \end{subfigure}
    \hfill
    \begin{subfigure}[b]{0.45\textwidth}
        \includegraphics[width=\textwidth]{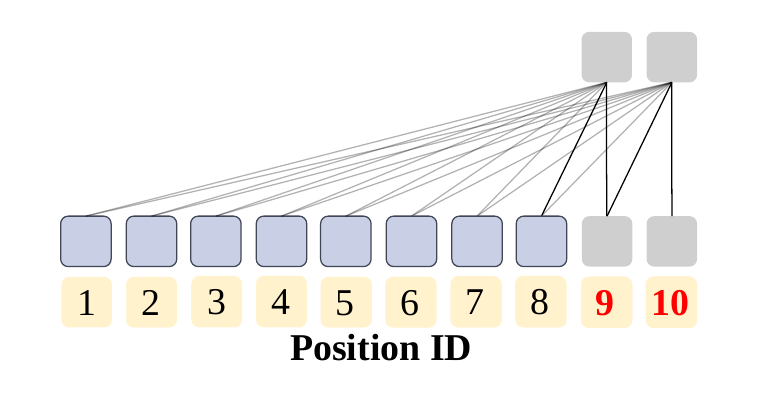}
        \caption{$\beta = 1.00$}
    \end{subfigure}
    \caption{Without modifying the attention mask, the bias factor $\beta$ is used to control the Positional Encodings (PEs) of learnable tokens. Subfigures (a) to (d) correspond to the changes in the PEs of learnable tokens when $\beta$ from $\{0.25, 0.50, 0.75, 1.00\}$.}
    \label{fig:Position_Encoding_Bias}
\end{figure}
To analyze the impact of \textit{position bias} on methods based on learnable tokens, we define the \textit{position bias factor} $\beta \in [0, 1]$, which is used to determine the starting Positional Encoding (PE) $pos_{\text{insert}}$ of these learnable tokens:
\begin{equation}
    pos_{\text{insert}} = \lfloor L \cdot \beta \rfloor+1 \, ,
\end{equation}
where $L$ denotes length of the input sequence. 

Consequently, the $K$ learnable tokens are assigned as a contiguous segment with position IDs $\{pos_{\text{insert}}, \dots, pos_{\text{insert}} + K - 1\}$, as depicted in Figure~\ref{fig:Position_Encoding_Bias}.

\section{Theoretical Analysis of Position Bias Mitigation}
\label{apx:theo}

\subsection*{Overview}
We establish three properties of the semantic  aggregation operator $\mathcal{A}$:
\textbf{(i)}~permutation invariance (Theorem.~\ref{thm:perm}), which structurally eliminates
ordering-dependent position bias~\citep{DBLP:journals/tacl/LiuLHPBPL24};
\textbf{(ii)}~positional noise attenuation at rate
$\mathcal{O}(|\mathcal{G}_k|^{-1})$ (Theorem.~\ref{thm:attn}), which bounds residual
variance after compression; and
\textbf{(iii)}~equivalence to a query-adaptive Nyström approximation
(Proposition.~\ref{prop:nystrom}), which connects the compression to kernel methods~\citep{DBLP:conf/nips/WilliamsS00,DBLP:conf/aaai/XiongZCTFLS21}.
Section~\ref{sec:pe_examples} instantiates these results on sinusoidal and RoPE encodings.

\subsection{Preliminaries}

Let $H = \{h_i\}_{i=1}^n \subset \mathbb{R}^d$ be the encoder hidden states,
decomposed as
\begin{equation}
    h_i = s_i + p_i,
\end{equation}
            where $s_i$ represents the position-invariant semantic component, while $p_i$ denotes the positional component. The semantic aggregation operator $\mathcal{A}$ with compression rate
$\tau$ produces $K = max(2, \lceil n/\tau \rceil)$ 

aggregated tokens:

\begin{equation}
    \alpha_i = \frac{\exp(w_i)}{\sum_{j \in \mathcal{G}_k} \exp(w_j)}, \quad \tilde{h}_k = \sum_{i \in \mathcal{G}_k} \alpha_i \cdot h_i,
    \label{eq:zk}
\end{equation} 
where $\mathcal{G}_k = \{i : k = \arg\max_j {v}(i,j)\}$ is the
Voronoi partition~\citep{DBLP:journals/tit/Lloyd82} induced by centroids $\{c_k\}$, with ${w}_i$ as defined in Eq.~\eqref{eq:weight}.

\subsection{Permutation Invariance}
\begin{theorem}[Permutation Invariance]
\label{thm:perm}
For any permutation $\sigma\in S_n$, where $S_n$ denotes the symmetric group
of all $n!$ bijections on $\{1,\dots,n\}$,
\begin{equation}
  \mathcal{A}((h_1,\dots,h_n)) = \mathcal{A}((h_{\sigma(1)},\dots,h_{\sigma(n)})),
\end{equation}
where $\mathcal{A}:\mathbb{R}^{n\times d}\to\mathbb{R}^{K\times d}$ is the
semantic aggregation operator defined in Eq.~\eqref{eq:zk}, and
$(h_{\sigma(1)},\dots,h_{\sigma(n)})$ denotes the input sequence reindexed
under $\sigma$.
\end{theorem}
\begin{proof}
The centroid $c_k$ is computed as a symmetric function of the input
multiset $\{\!\{h_1,\dots,h_n\}\!\}$, so it is unaffected by any permutation
$\sigma\in S_n$. Given fixed centroids, compression group membership $i\in\mathcal{G}_k$
is determined solely by the value of $h_i$, not by the index $i$. A
permutation $\sigma$ therefore preserves the content multiset
$\mathcal{H}_k=\{\!\{h_i:i\in\mathcal{G}_k\}\!\}$ of every compression group.
Since $\tilde{h}_k$ is a symmetric (softmax-weighted) function of
$\mathcal{H}_k$, it is invariant to the
ordering of its elements, and the conclusion follows.
\end{proof}

\subsection{Positional Noise Attenuation}

Let $\tilde{s}_k = \sum_{i\in\mathcal{G}_k}\alpha_i s_i$ be the target semantic
aggregation representation, where $\{\alpha_i\}_{i\in\mathcal{G}_k}$ are non-negative weights satisfying
$\sum_{i\in\mathcal{G}_k}\alpha_i = 1$. Define the positional residual
$\epsilon_k = \tilde{h}_k - \tilde{s}_k = \sum_{i\in\mathcal{G}_k}\alpha_i p_i$.

\begin{theorem}[Noise Attenuation]
\label{thm:attn}
Assume $\{p_i\}_{i\in\mathcal{G}_k}$ are independent with
$\mathbb{E}[p_i]=0$ and $\mathrm{Cov}(p_i)=\Sigma_p$ for all $i\in\mathcal{G}_k$.
Then
\begin{equation}
    \mathbb{E}\|\tilde{h}_k-\tilde{s}_k\|^2
    = \Bigl(\textstyle\sum_{i\in\mathcal{G}_k}\alpha_i^2\Bigr)\,\mathrm{tr}(\Sigma_p).
\end{equation}
In particular, under uniform weights $\alpha_i = 1/|\mathcal{G}_k|$, justified
by semantic redundancy within a high-density cluster, this reduces to
\begin{equation}
    \mathbb{E}\|\tilde{h}_k-\tilde{s}_k\|^2 = \frac{\mathrm{tr}(\Sigma_p)}{|\mathcal{G}_k|}.
\end{equation}
\end{theorem}

\begin{proof}
Since $\mathbb{E}[p_i]=0$ for all $i$, linearity of expectation gives
$\mathbb{E}[\epsilon_k] = \sum_{i\in\mathcal{G}_k}\alpha_i\,\mathbb{E}[p_i] = 0$,
so the bias term vanishes and
\begin{equation*}
    \mathbb{E}\|\epsilon_k\|^2
    = \mathbb{E}\|\epsilon_k - \mathbb{E}[\epsilon_k]\|^2
    = \mathrm{tr}\bigl(\mathrm{Cov}(\epsilon_k)\bigr).
\end{equation*}
By independence of $\{p_i\}$ and the identical covariance structure,
\begin{equation*}
    \mathrm{Cov}(\epsilon_k)
    = \sum_{i\in\mathcal{G}_k}\alpha_i^2\,\mathrm{Cov}(p_i)
    = \Bigl(\sum_{i\in\mathcal{G}_k}\alpha_i^2\Bigr)\Sigma_p,
\end{equation*}
and thus
$\mathrm{tr}(\mathrm{Cov}(\epsilon_k))
= \bigl(\sum_{i\in\mathcal{G}_k}\alpha_i^2\bigr)\,\mathrm{tr}(\Sigma_p)$.
Under uniform weights $\alpha_i = 1/|\mathcal{G}_k|$, we have
$\sum_{i}\alpha_i^2 = 1/|\mathcal{G}_k|$, yielding the stated bound.
\end{proof}

\begin{remark}[Correlated noise]
If positional components are stationary with integer indices $i,j\in\mathcal{G}_k$
and $\mathrm{Cov}(p_i, p_j) = \Sigma_p\,\rho(|i-j|)$ for some scalar
function $\rho:\mathbb{Z}_{\geq 0}\to[-1,1]$ with $\rho(0)=1$ (Toeplitz
structure), then under uniform weights
\begin{equation*}
    \mathrm{Cov}(\epsilon_k)
    = \frac{\Sigma_p}{|\mathcal{G}_k|^2}
      \sum_{i,j\in\mathcal{G}_k}\rho(|i-j|).
\end{equation*}
A sufficient condition for preserving the $\mathcal{O}(|\mathcal{G}_k|^{-1})$
rate of Theorem~\ref{thm:attn} is absolute summability of the correlation
function, i.e., $\sum_{\ell=0}^{\infty}|\rho(\ell)|<\infty$, which holds in
particular when $\rho(\ell)$ decays exponentially. Under this condition,
$\frac{1}{|\mathcal{G}_k|}\sum_{i,j}\rho(|i-j|) = \mathcal{O}(1)$, so
$\mathrm{tr}(\mathrm{Cov}(\epsilon_k)) = \mathcal{O}(|\mathcal{G}_k|^{-1})\,\mathrm{tr}(\Sigma_p)$,
and the off-diagonal contributions remain lower-order.
\end{remark}

\subsection{Nyström Approximation Perspective}

\begin{proposition}[Structural Correspondence to Nyström]
\label{prop:nystrom}
Define the kernel matrices $\mathbf{K}_H\in\mathbb{R}^{n\times n}$,
$\mathbf{K}_{HZ}\in\mathbb{R}^{n\times K}$, and
$\mathbf{K}_{ZZ}\in\mathbb{R}^{K\times K}$ with entries
$(\mathbf{K}_H)_{ij}=s_{i,j}$, $(\mathbf{K}_{HZ})_{ik}=s_{i,k}$,
and $(\mathbf{K}_{ZZ})_{kl}=s_{k,l}$, respectively,
where the similarity $s_{\cdot,\cdot}$ is given by Eq.~\eqref{eq:similarity}.
Then the aggregated representations $\{\tilde{h}_k\}$ satisfy
\begin{equation}
    \tilde{H} = A \cdot H,
    \qquad A_{ki} = \alpha_i^{(k)}\cdot\mathbf{1}[i\in\mathcal{G}_k],
\end{equation}
where the assignment matrix $A \in \mathbb{R}^{K \times n}$
shares the similarity-weighted structure of the Nyström projection
$\mathbf{K}_{ZZ}^{-1}\mathbf{K}_{ZH}$ in the sense that both operators
assign weight to token $i$ under landmark $k$ proportionally to $s_{k,i}$:
the former via softmax normalization within $\mathcal{G}_k$,
the latter via global kernel inversion.
\end{proposition}

\begin{proof}
By definition, $\tilde{h}_k = \sum_{i \in \mathcal{G}_k} \alpha_i h_i$,
so the full output $\tilde{H} \in \mathbb{R}^{K \times d}$ is a linear map
$A \cdot H$ where $A_{ki} = \alpha_i^{(k)} \cdot \mathbf{1}[i \in \mathcal{G}_k]$.
The Nyström approximation reconstructs
$\mathbf{K}_H \approx \mathbf{K}_{HZ}\mathbf{K}_{ZZ}^{-1}\mathbf{K}_{ZH}$,
where the per-landmark reconstruction weight for token $i$ under landmark $k$
is given by $(\mathbf{K}_{ZZ}^{-1}\mathbf{K}_{ZH})_{ki}$, which is
proportional to $s_{k,i}$ when $\mathbf{K}_{ZZ}$ is well-conditioned.
Since $\alpha_i^{(k)}$ is derived from the same similarity $s_{k,i}$
via softmax normalization over $\mathcal{G}_k$ (Eq.~\eqref{eq:zk}),
$A$ shares the same similarity-weighted structure restricted to each
partition cell, establishing the structural correspondence.
\end{proof}

\begin{remark}[Query-Adaptive Landmarks and Approximation Quality]
\label{rem:nystrom_quality}
The Nyström approximation error for a rank-$K$ reconstruction is known
to depend on how well the landmark subspace
$\mathrm{span}\{c_k\}$ captures the dominant eigendirections of
$\mathbf{K}_H$~\citep{DBLP:conf/nips/WilliamsS00}.
Fixed-position landmarks (e.g., uniform or stride-based sampling)
align with positional eigendirections that may be orthogonal to the
query-relevant subspace.
In contrast, SeCo selects landmarks by maximizing similarity between $c_k$ and $\bar{q}$,
concentrating $\mathrm{span}\{c_k\}$ on query-relevant eigendirections
of $\mathbf{K}_H$.
Under the standard Nyström error decomposition~\citep{DBLP:conf/aaai/XiongZCTFLS21},
this implies lower reconstruction error on the query-relevant subspace.

\end{remark}

\subsection{Instantiation on Sinusoidal and RoPE Encodings}
\label{sec:pe_examples}

We instantiate Theorems~\ref{thm:perm}--\ref{thm:attn} under the query-aware
semantic aggregation of SeCo, focusing on the dominant behavior of positional residuals.

\subsubsection{Sinusoidal Positional Encoding}

In sinusoidal encoding, the positional component is additive and independent
of semantics.

\paragraph{Permutation Invariance.}
Since both semantic assignment and consistency merging depend only on semantic similarity and
query relevance, aggregation is permutation invariant.

\paragraph{Noise Attenuation.}

\begin{theorem}[Sinusoidal Case]
Under bounded and normalized weights, the positional residual in each compression group
satisfies
\begin{equation}
\mathbb{E}\|\epsilon_k\|^2 = \mathcal{O}(|\mathcal{G}_k|^{-1}).
\end{equation}
\end{theorem}

\begin{proof}[Sketch]
The residual takes the form
\begin{equation}
\epsilon_k = \sum_{i\in\mathcal{G}_k} \alpha_i \sin(i\omega),
\end{equation}
which is a weighted average of bounded oscillatory terms. Since $\sin(i\omega)$
is approximately zero-mean over phases, and SeCo assigns tokens based on semantic
similarity rather than position, the indices within $\mathcal{G}_k$ can be viewed
as approximately unordered samples.

Expanding the second moment,
\begin{equation}
\mathbb{E}\|\epsilon_k\|^2
= \sum_i \alpha_i^2 \mathbb{E}[\sin^2(i\omega)]
+ \sum_{i\neq j} \alpha_i \alpha_j \mathbb{E}[\sin(i\omega)\sin(j\omega)].
\end{equation}

The cross terms are suppressed due to oscillation and weak correlation between
phases, leaving the dominant contribution proportional to $\sum_i \alpha_i^2$.
Under bounded weights, $\sum_i \alpha_i^2 = \mathcal{O}(|\mathcal{G}_k|^{-1})$,
which yields the stated result.
\end{proof}

\subsubsection{Rotary Position Embedding (RoPE)}

RoPE introduces position-dependent rotations, leading to a residual that
depends on both content and position.

\paragraph{Permutation Invariance.}
Permutation invariance does not hold strictly, but is partially mitigated by
semantic aggregation.

\paragraph{Noise Attenuation and Bias.}

\begin{theorem}[RoPE Case]
Assume bounded hidden states and smooth variation of rotations. Then the
residual satisfies
\begin{equation}
\|\epsilon_k^{(\mathrm{R})}\|
= \mathcal{O}(\Delta\theta_k) + \mathcal{O}(|\mathcal{G}_k|^{-1/2}),
\end{equation}
where $\Delta\theta_k$ is the angular spread within the compression group.
\end{theorem}

\begin{proof}[Sketch]
The residual can be written as
\begin{equation}
\epsilon_k^{(\mathrm{R})}
= \sum_{i\in\mathcal{G}_k} \alpha_i (\mathbf{R}_{\theta,i} - \mathbf{I}) h_i.
\end{equation}

We decompose it into a mean and fluctuation term:
\begin{equation}
\epsilon_k^{(\mathrm{R})}
= (\bar{\mathbf{R}}_k - \mathbf{I}) \bar{h}_k
+ \sum_{i\in\mathcal{G}_k} \alpha_i (\mathbf{R}_{\theta,i} - \bar{\mathbf{R}}_k) h_i,
\end{equation}
where $\bar{\mathbf{R}}_k = \sum_i \alpha_i \mathbf{R}_{\theta,i}$ and
$\bar{h}_k = \sum_i \alpha_i h_i$.

The first term represents a bias induced by the average rotation, whose magnitude
scales with the variation of rotation angles within the compression group, i.e.,
$\mathcal{O}(\Delta\theta_k)$. The second term aggregates zero-mean fluctuations
around the mean rotation and behaves like a weighted average of weakly correlated
terms, yielding a decay rate of $\mathcal{O}(|\mathcal{G}_k|^{-1/2})$. Combining both terms gives the stated bound.
\end{proof}

\section{Stress Test}
\label{apx:stress_test}
To further evaluate the robustness of methods, we conduct a stress test experiment. In this experiment, SeCo, COMI, and RAM all adopted a random compression rate strategy during the training phase: in each training iteration, we randomly selected a compression rate $\tau$ from the set $\{2, 4, 8, 16, 32\}$ (\texttt{LLaMA-3.2-1B-Instruct} as backbone). As shown in Figure \ref{fig:pressure-results}, while increasing the compression rate from 2 to 32 leads to varying degrees of performance degradation across methods, \emph{SeCo consistently maintains superior performance compared to COMI and RAM on both in-domain and out-of-domain settings.} This performance gap indicates that SeCo possesses stronger robustness under various compression constraints.

\begin{figure}[t]
  \centering
     \includegraphics[width=\textwidth]{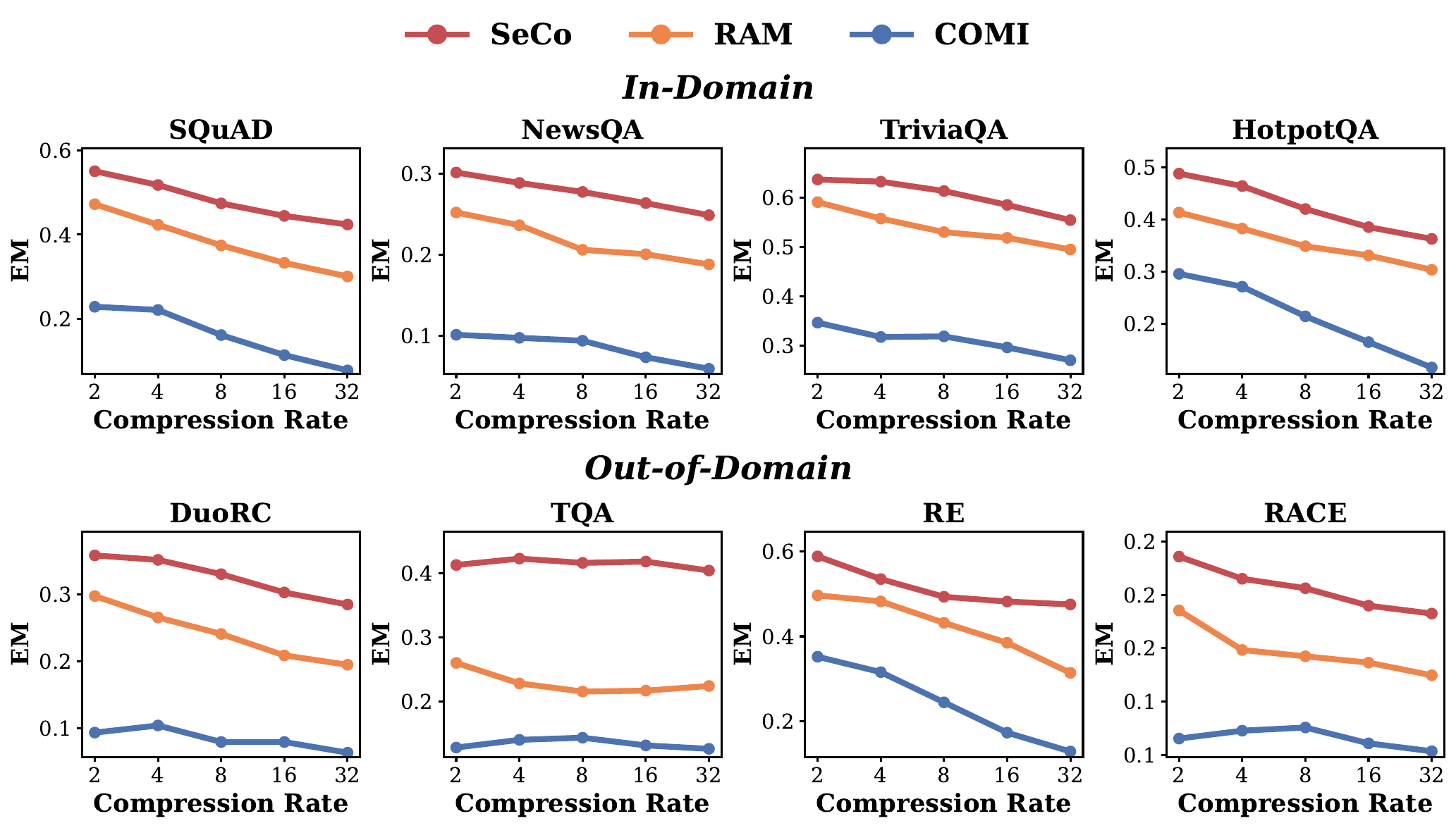}
     \caption{Exact Match (EM) across both in-domain and out-of-domain QA datasets at different compression rates $\{2,4,8,16,32\}$.}
      \label{fig:pressure-results}
\end{figure}

\section{Dataset Details}
\label{apx:dataset_details}
\paragraph{SQuAD.} SQuAD~\citep{DBLP:conf/emnlp/RajpurkarZLL16} is a reading comprehension benchmark derived from Wikipedia, where models must locate answer spans directly within the provided passage to respond to given questions.

\paragraph{NewsQA.} NewsQA~\citep{DBLP:conf/rep4nlp/TrischlerWYHSBS17} is a reading comprehension dataset built on CNN news articles, posing questions that frequently require inference beyond surface-level text matching.

\paragraph{TriviaQA.} TriviaQA~\citep{DBLP:conf/acl/JoshiCWZ17} is a large-scale QA dataset pairing independently authored trivia questions with evidence documents sourced from Wikipedia and the web.

\paragraph{SearchQA.} SearchQA~\citep{DBLP:journals/corr/DunnSHGCC17} is a QA dataset that pairs Jeopardy! clues with web-retrieved snippets as context, challenging models to synthesize information across multiple documents.

\paragraph{HotpotQA.} HotpotQA~\citep{DBLP:conf/emnlp/Yang0ZBCSM18} is a multi-hop QA dataset requiring models to jointly reason over several supporting documents, with annotated supporting facts provided to facilitate explainability.

\paragraph{NaturalQuestions.} NaturalQuestions~\citep{DBLP:journals/tacl/KwiatkowskiPRCP19} is an open-domain QA benchmark composed of real user queries issued to the Google search engine. Models must extract both short and long-form answers from Wikipedia articles, making it a realistic testbed for evaluating reading comprehension under natural information-seeking conditions.

\paragraph{DROP.} DROP~\citep{DBLP:conf/naacl/DuaWDSS019} is a reading comprehension benchmark emphasizing discrete reasoning over text, where answering questions may involve numerical operations such as addition, subtraction, counting, or sorting.

\paragraph{BioASQ.} BioASQ~\citep{DBLP:journals/bmcbi/TsatsaronisBMPZ15} is a biomedical QA dataset featuring expert-curated questions paired with PubMed abstracts, demanding specialized domain knowledge to generate exact or abstractive answers.

\paragraph{DuoRC.ParaphraseRC.} DuoRC.ParaphraseRC~\citep{DBLP:conf/acl/KhapraSSA18} is a reading comprehension dataset grounded in movie plot summaries, where the ParaphraseRC variant supplies a paraphrased context document, substantially reducing lexical overlap with the original question.

\paragraph{TextbookQA.} TextbookQA~\citep{DBLP:conf/cvpr/KembhaviSSCFH17} is a QA dataset drawn from middle school science textbooks, requiring comprehension of both textual descriptions and accompanying diagrams to answer curriculum-based questions.

\paragraph{RelationExtraction.} RelationExtraction~\citep{DBLP:conf/conll/LevySCZ17} casts relation extraction as a reading comprehension problem, mapping each relation type to a natural language question and requiring the model to extract the corresponding answer span from the input sentence.

\paragraph{RACE.} RACE~\citep{DBLP:conf/emnlp/LaiXLYH17} is a multiple-choice reading comprehension dataset sourced from English exams administered in China, encompassing diverse topics and demanding skills such as reasoning, summarization, and commonsense understanding.

\paragraph{DialogSum.} DialogSum~\citep{DBLP:conf/acl/ChenLCZ21} is a dialogue summarization dataset pairing everyday spoken conversations with human-written summaries, with the goal of distilling key information and communicative intent from conversational text.

\paragraph{NarrativeQA.} NarrativeQA~\citep{DBLP:journals/tacl/KociskySBDHMG18} is a reading comprehension dataset grounded in full-length books and movie scripts, where questions are crafted by annotators who only see human-written summaries, necessitating deep narrative understanding to answer correctly. We truncate each input to at most 32K tokens.

\section{Training Details}
\label{apx:train_details}
\noindent
\begin{paracol}{2}

We fine-tune SeCo with bf16 precision using data sampled from the MRQA benchmark \citep{DBLP:conf/acl-mrqa/FischTJSCC19}. For the QA task, we construct our training set by randomly selecting 4,000 samples per subset. LLaMA-3.2-1B-Instruct\hyperlink{fn1}{$^1$} is trained on 4 NVIDIA RTX 4090 GPUs, while Qwen3-4B-Instruct\hyperlink{fn2}{$^2$} is trained on 2 NVIDIA A800 GPUs. For the summarization task and long-context QA, we fine-tune Qwen3-4B-Instruct\hyperlink{fn2}{$^2$} on DialogSum and NarrativeQA using 2 NVIDIA A800 GPUs and 32 NVIDIA H200 GPUs, respectively. \textbf{To ensure a fair comparison, all baselines are trained using the same datasets and hyperparameters.} Detailed training configurations are summarized in Table~\ref{tab:hyperparams}.

\switchcolumn
\captionof{table}{Hyperparameters for fine-tuning.}
\label{tab:hyperparams}
\begin{tabularx}{\linewidth}{l >{\centering\arraybackslash}X}
    \toprule
    \textbf{Hyperparameter} & \textbf{Value} \\
    \midrule
    Per Device Batch Size & 1 \\
    Optimizer & AdamW \\
    Betas & $(0.9, 0.999)$ \\
    Learning Rate & $1 \times 10^{-5}$ \\
    LoRA Rank & 128 \\
    LoRA Alpha & 32 \\
    LoRA Dropout & 0.05 \\
    Target Modules & all-linear \\
    Warmup Ratio & 0.1 \\
    Weight Decay & 0.01 \\
    Scheduler & cosine \\
    Precision & bf16 \\
    \bottomrule
\end{tabularx}
\end{paracol}

\hypertarget{fn1}{}
\footnotetext[1]{https://huggingface.co/meta-llama/Llama-3.2-1B-Instruct}
\hypertarget{fn2}{}
\footnotetext[2]{https://huggingface.co/Qwen/Qwen3-4B-Instruct-2507}

\section{Limitation}
\phantomsection
\label{sec:limitations}

While SeCo effectively mitigates position bias and enhances compression quality through semantic space clustering, it still possesses certain limitations. Specifically, SeCo introduces a query-aware clustering process during the compression phase. Although this process is lightweight compared to backbone model inference and the overhead of other compression methods (see Tab.~\ref{tab:eval_lat}), it still introduces additional computational costs during the end-to-end generation stage. Notably, however, the process of assigning non-centroid tokens to cluster centers is inherently parallelizable, which partially alleviates the aforementioned efficiency concerns. 

\section{Ethic Statement}
\label{apx:Ethic Statement}
This work presents SeCo, a framework based on the encoder-decoder architecture that enables task-aware context compression by performing clustering in the semantic space.
Regarding the broader societal impact, SeCo enhances the efficiency and efficacy of context compression within machine learning pipelines. We have evaluated the potential implications of this framework and concluded that it does not introduce novel ethical risks or societal harms beyond those inherent in existing context compression technologies. We remain committed to ethical AI development and will continue to monitor the implications of our work as it is integrated into broader applications.

\section{Language Model Usage Statement}
\label{apx:LLM Usage Statement}
Throughout the development of this paper, we integrated a large language model into our iterative drafting process to assist in synthesizing our technical notes into a coherent narrative. The model served as a supportive tool for structuring arguments and ensuring consistent terminology, rather than generating the underlying research content. We exercised complete editorial oversight, conducting rigorous manual reviews of all AI-generated suggestions to ensure scientific precision. The intellectual framework, the validity of the proofs, and the integrity of all empirical findings remain exclusively the product of the authors.

\end{document}